\documentclass[11pt,a4paper]{article}
\usepackage[margin=2.5cm]{geometry}
\usepackage{cmap}
\usepackage[T1]{fontenc}
\usepackage{amsmath,amssymb,amsthm}
\usepackage{booktabs}
\usepackage{graphicx}
\usepackage{enumitem}
\usepackage[colorlinks=true,linkcolor=black,citecolor=black,urlcolor=black]{hyperref}

\newtheorem{proposition}{Proposition}
\DeclareMathOperator{\corr}{corr}
\DeclareMathOperator{\artanh}{artanh}

\title{Trail-Weighted Influence Scores for Approximate Outcome Prediction\\in Bayesian Networks}
\author{Sahil Kumar Rajput \and Randhir Kumar}
\date{October 2026}

\begin{document}
\maketitle

\begin{abstract}
We study trail-weighted pruning of evidence for Bayesian-network prediction. A traversal enumerates the active trails from each evidence node to the outcome, scores each trail by the product of its edge strengths, and drops evidence whose score falls below a threshold. For collider-free networks (chains and forks) over binary variables, we show that the trail score equals the exact correlation between a trail's endpoints, that a chain and a fork with equal edge strengths have identical joint distributions, and that the pruned posterior is within a computable sum of per-trail terms of the exact posterior when each dropped trail lies in its own branch of the outcome node. We prove this bound, check it exhaustively on all small trees, and show that it fails when dropped trails share a branch. In simulations on random trees, thresholding trail scores gives 1.4 to 13 times lower error than random pruning of the same size. We also apply the score, as an unproven heuristic, to seven bnlearn networks with colliders and loops. Keeping the top half of the active evidence by trail score gives lower error than random pruning on all seven, but it is not clearly better than a mutual-information ranking, and scoring costs more than the inference it saves on six of the seven, so we show no speedup. Models trained on trail, graph and information features and tested on networks they never saw do not beat the trail score or mutual information, but trail features carry most of what they use and clearly outperform generic graph features. A bound for general networks is left to future work.
\end{abstract}

\section{Introduction}

Predicting an outcome $Y$ from partial evidence $E$ in a Bayesian network requires the posterior $P(Y \mid E)$, which is NP-hard to compute exactly in general (Cooper, 1990). A user also wants to know which evidence actually drove the prediction, and ideally to compute the answer from that evidence alone.

Standard relevance pruning uses d-separation. Evidence that is d-separated from $Y$ by other observed nodes can be removed with no error (Geiger, Verma and Pearl, 1990; Shachter, 1998). This is exact but coarse: every evidence node with an active trail to $Y$ is kept, even when that trail is a long chain of weak dependencies whose effect on $Y$ is negligible.

We rank the remaining evidence by trail strength. A traversal from $Y$ enumerates the active trails, each trail is scored by composing the strengths of its edges, and evidence whose best trail scores below a threshold is discarded. The score needs no inference, and it names the trail responsible for each piece of evidence. Earlier approximations delete weak links from the network (Kj{\ae}rulff, 1994; Choi, Chan and Darwiche, 2005). We instead prune the evidence set and leave the network intact.

This paper makes six contributions. The first four concern collider-free networks (chains and forks) over binary variables, and the last two are empirical:
\begin{enumerate}[leftmargin=*]
\item \textbf{Exact trail scores.} The product of edge strengths along a trail equals the exact correlation between its endpoints, and mutual information is a monotone function of it.
\item \textbf{Chain--fork equivalence.} A chain and a fork with equal edge strengths have identical joint distributions, so every trail score treats them identically, which illustrates Markov equivalence.
\item \textbf{A pruning bound with a stated condition.} When each dropped trail lies in its own branch of $Y$, the error of the pruned posterior is bounded by a sum over dropped trails. The bound is nearly tight for weak trails. We show empirically that it can fail when dropped trails share a branch.
\item \textbf{A simulation study with exact inference.} d-separation pruning is exact, and trail thresholding beats random pruning of the same size across all thresholds tested.
\item \textbf{A benchmark study.} On seven bnlearn networks, with colliders and loops handled heuristically, trail pruning beats random pruning on every network and hop distance on four of seven, is not clearly better than a mutual-information ranking, and shows no speedup.
\item \textbf{A test of learned calibration.} Models trained on trail, graph and information features and tested on unseen networks do not beat the trail score or mutual information. Trail features add clear predictive value over generic graph features and make up 64\% of the model's attributions.
\end{enumerate}
A bound for colliders and loopy networks is outside this paper and is described under future work.

\section{Background and notation}

A Bayesian network is a directed acyclic graph over variables whose joint distribution factorizes as the product of each variable's distribution given its parents. We fix an outcome node $Y$ and a set $E$ of observed evidence nodes with observed values.

A \emph{trail} is a path in the graph that ignores edge directions. A trail from an evidence node to $Y$ is \emph{active} given the other evidence if every chain ($A \to B \to C$) and every fork ($A \leftarrow B \to C$) on it has an unobserved middle node, and every collider ($A \to B \leftarrow C$) has the middle node or one of its descendants observed. Evidence with no active trail to $Y$ is d-separated from $Y$ and can be removed without changing $P(Y \mid E)$.

This paper treats collider-free networks. These are rooted directed trees, where every node has at most one parent, so every trail is built from chains and forks. An observed interior node then blocks a trail, and a trail is active exactly when none of its interior nodes is observed.

\paragraph{Edge strength.} For a pair of variables $(X,Z)$, the maximal correlation $\rho(X;Z)$ is the largest correlation between any functions $f(X)$ and $g(Z)$. It equals the second-largest singular value of the matrix with entries $P(x,z)/\sqrt{P(x)P(z)}$ (R\'enyi, 1959).

\paragraph{Model class.} The root is uniform on $\{0,1\}$. Each other node copies its parent and flips with probability $\varepsilon \in (0, 1/2)$. Every marginal is then uniform, and an edge with flip probability $\varepsilon$ has strength $1-2\varepsilon$. The strength of a trail $T$ is the product over its edges:
\begin{equation}
\rho(T) = \prod_{e \in T} \rho_e, \qquad \rho_e = 1 - 2\varepsilon_e .
\end{equation}

\paragraph{Error measures.} The pruned posterior $q$ is compared with the exact posterior $p$ by total variation $|p-q|$ and by the Kullback--Leibler divergence of the two Bernoulli distributions.

\section{Method}

Trail-weighted pruning runs in three steps and needs no inference until the last one.
\begin{enumerate}[leftmargin=*]
\item \textbf{Traverse.} Run a depth-first search from $Y$ over the network skeleton. For every node reached, record the product of the edge strengths along its trail, and whether the trail is active, meaning no interior node is observed.
\item \textbf{Prune.} Keep the evidence nodes whose trail is active and whose score is at least a threshold $\tau$. Evidence on blocked trails is dropped exactly, by d-separation. Active evidence below $\tau$ is dropped approximately.
\item \textbf{Infer.} Compute $P(Y \mid \text{kept evidence})$ by exact inference on the reduced evidence set.
\end{enumerate}
In a tree each node has a single trail to $Y$, so the traversal visits each node once and costs time linear in the number of nodes. On general graphs the number of trails can grow exponentially, so enumeration would need caps on trail length and count. We do not treat that case here.

For a set $D$ of dropped active evidence nodes, with trail scores $\rho_d$, we study the additive error term
\begin{equation}
B(D) = \frac{1}{4} \sum_{d \in D} \ln \frac{1+\rho_d}{1-\rho_d}.
\end{equation}
Section~\ref{sec:theory} shows when $|p-q| \le B(D)$ is guaranteed.

\paragraph{Choice of strength.} We use maximal correlation with product composition because it is exact for binary symmetric edges and contracts along chains for general distributions. The alternative of normalized mutual information, composed by the minimum or the product, is a natural choice for general conditional distributions and is deferred.

\paragraph{Extension to general networks.} In a network with colliders and loops, a node can have many trails to $Y$, and a trail is active only if every collider on it, or a descendant of that collider, is observed. We enumerate simple trails of at most 7 edges, keep at most 300 per evidence node, test each against the collider rule, and give the evidence node the best product score among its active trails. Nodes with no active trail within the cap score zero. Edge strength is the maximal correlation of the parent--child pair, computed once per network by exact inference. This is a heuristic: Propositions 1 to 3 do not cover it, and the product scores of trails that share nodes are not additive.

\section{Theory}\label{sec:theory}

Three results hold in the collider-free binary model of the Background section: an exact trail score, a chain--fork equivalence, and an error bound under a stated condition. Code each variable as $S = 2X-1$, so $S$ takes the values $-1$ and $+1$.

\begin{proposition}[The trail score is exact]
For any trail $T$ between nodes $X_0$ and $X_k$, the correlation of the endpoints equals the trail strength, and their mutual information is a strictly increasing function of it:
\begin{equation}
\corr(X_0, X_k) = \rho(T), \qquad I(X_0; X_k) = \ln 2 - h\!\left(\frac{1-\rho(T)}{2}\right),
\end{equation}
where $h$ is the binary entropy in nats.
\end{proposition}
\begin{proof}
Removing any node splits a tree, so the variables along a path form a Markov chain. Along each edge $\mathbb{E}[S_{\text{child}} \mid S_{\text{parent}}] = \rho_e S_{\text{parent}}$, and by symmetry and uniform marginals the same holds in the other direction. Iterating the conditional expectation along the path gives $\mathbb{E}[S_k \mid S_0] = \rho(T) S_0$, and therefore $\mathbb{E}[S_0 S_k] = \rho(T)$. The pair $(X_0, X_k)$ has uniform marginals and agrees with probability $(1+\rho(T))/2$, so it is a binary symmetric pair of strength $\rho(T)$, which gives the mutual information formula.
\end{proof}

For equal edge strengths $\rho$, a chain of $k$ edges has score $\rho^k$, so dependence decays geometrically with length. For general distributions the product is an upper bound: maximal correlation satisfies $\rho(X;Z) \le \rho(X;Y)\,\rho(Y;Z)$ along any Markov chain $X - Y - Z$ (Witsenhausen, 1975).

\begin{proposition}[Chains and forks are indistinguishable]
For edge strengths $\rho_A$ and $\rho_C$, the chain $A \to B \to C$ with a uniform root $A$ and the fork $A \leftarrow B \to C$ with a uniform root $B$ have the same joint distribution. Hence $\corr(A;C) = \rho_A \rho_C$ in both, and $A$ is independent of $C$ given $B$ in both.
\end{proposition}
\begin{proof}
Let $W$ denote the symmetric channel probability $P(\text{output} \mid \text{input})$, so $W(x \mid y) = W(y \mid x)$. The two joint distributions are
\begin{equation*}
P_{\text{chain}}(a,b,c) = \tfrac12\, W_A(b \mid a)\, W_C(c \mid b), \qquad
P_{\text{fork}}(a,b,c) = \tfrac12\, W_A(a \mid b)\, W_C(c \mid b),
\end{equation*}
and they agree because $W_A$ is symmetric.
\end{proof}
Because the full joint is identical, no score computed from the distribution can tell the two structures apart, which is Markov equivalence in the simplest case.

\begin{proposition}[Error of dropping isolated trails]\label{prop:bound}
A \emph{branch} of $Y$ is a connected component of the tree once $Y$ is removed. Let $D$ be a set of dropped active evidence nodes, and suppose each $d \in D$ is the only active evidence node in its branch. Let $p$ be the exact posterior $P(Y=1 \mid E)$ and $q$ the posterior given the kept active evidence. Then
\begin{equation}
|p-q| \le B(D) = \frac14 \sum_{d \in D} \ln \frac{1+\rho_d}{1-\rho_d}.
\end{equation}
\end{proposition}
\begin{proof}
Let $A$ be the set of all active evidence nodes, $A' = A \setminus D$ the kept ones, and $e_S$ the observed values of a set $S$ of nodes. Both posteriors are written as logistic functions of log-odds, and the proof compares those.

\emph{Step 1: only active evidence matters.} Let $b$ be a blocked evidence node. The path from $b$ to $Y$ in the tree is unique, and it contains another observed node, since otherwise $b$ would be active. The observed node on it nearest to $Y$ has no observed node between itself and $Y$, so it is active, and it is a chain or fork node on the path. Hence $Y$ is d-separated from the blocked nodes given $A$, and $p = P(Y=1 \mid e_A)$.

\emph{Step 2: the likelihood factorizes by branch.} Given $Y$ the branches are independent, so
$\mathrm{logit}\,P(Y=1 \mid e_A) = \mathrm{logit}\,P(Y=1) + \sum_\beta \lambda_\beta$, with $\lambda_\beta = \ln\frac{P(e_{A\cap\beta} \mid Y=1)}{P(e_{A\cap\beta} \mid Y=0)}$ summed over the branches $\beta$ of $Y$. Each $d \in D$ is the only active node in its branch, so dropping it removes exactly the term $\lambda_d = \ln\frac{P(e_d \mid Y=1)}{P(e_d \mid Y=0)}$ and leaves every other term unchanged. The log-odds of $p$ and of $q = P(Y=1 \mid e_{A'})$ therefore differ by $\sum_{d\in D}\lambda_d$.

\emph{Step 3: each term is known.} By Proposition~1, $(Y, e_d)$ is a binary symmetric pair with uniform marginals and correlation $\rho_d$, so $P(e_d \mid Y=y) = (1+\rho_d s_d s_y)/2$ in the $\pm1$ coding, and $|\lambda_d| = \ln\frac{1+\rho_d}{1-\rho_d}$.

\emph{Step 4: Lipschitz.} The logistic function $\sigma$ has $\sigma' = \sigma(1-\sigma) \le \tfrac14$, so $|p-q| \le \tfrac14\bigl|\sum_{d\in D}\lambda_d\bigr| \le \tfrac14\sum_{d\in D}|\lambda_d| = B(D)$.
\end{proof}

\paragraph{Tightness.} With a single dropped node, no other evidence and a uniform prior, the exact error is $\rho/2$, while the bound is $\tfrac12 \artanh \rho$. Their ratio is $\rho/\artanh\rho = 1 - \rho^2/3 + O(\rho^4)$, so it tends to 1 as $\rho$ falls. The bound is therefore nearly tight for the weak trails that pruning removes.

\paragraph{Exhaustive verification.} We checked Proposition~\ref{prop:bound} by enumeration, independently of the proof. For every labelled tree on 4, 5 and 6 nodes (\texttt{verify\_prop3.py}), with $Y$ at a fixed node, edge strengths on a grid, every evidence set and every assignment of its values, and every non-empty subset $D$ of the active evidence dropped, we computed $p$ and $q$ exactly from the $2^n$ joint states. Of 23,917,216 cases in which each dropped node was alone in its branch, none violated the bound, and the largest ratio of error to bound was 0.99998. A second script (\texttt{bn\_prop3\_exhaustive.py}) covers all non-isomorphic trees on 3 to 7 nodes, every outcome node and random edge strengths: of 595,728 isolated cases, again none violated the bound (largest ratio 0.99999998). In the cases where dropped nodes share a branch, 198,040 of 4,970,080 cases in the first script violated it, with the error up to 1.74 times the bound. The smallest counterexample has four nodes: $Y$ joined to node 1 with strength 0.1, and nodes 2 and 3 both joined to node 1 with strengths 0.9 and 0.99. With nodes 2 and 3 observed at opposite values and node 3 dropped, the error is 0.086 and the bound is 0.050.

\paragraph{Where the bound does not apply.} If two active evidence nodes share a branch of $Y$, they are dependent given $Y$, their log-likelihood ratios no longer add, and Proposition~\ref{prop:bound} does not cover them. Trail scores are also non-additive whenever trails share nodes. We have no valid bound for this case, and it is the main open problem; the experiments below measure how often the additive term fails there.

\section{Synthetic experiments}

Trail-score pruning beats random pruning of the same size at every threshold, d-separation pruning is exact, and the additive bound holds whenever Proposition~\ref{prop:bound} applies.

\paragraph{Setup.} We use two regimes of random rooted trees over binary variables. The first has 12 nodes, with each parent drawn uniformly from earlier nodes (4,000 trials). The second has 14 nodes, with each parent drawn from the previous two nodes, which gives deeper trees (3,000 trials). Edge flip probabilities are uniform on $[0.05, 0.45]$. In each trial $Y$ is a random node, and 2 to $n-1$ evidence nodes are chosen at random, with values sampled from the exact joint distribution. Exact posteriors come from enumerating all $2^n$ states. Thresholds are $\tau \in \{0, 0.02, 0.05, 0.1, 0.2, 0.3, 0.5\}$. The random baseline keeps the same number of active evidence nodes as trail-score pruning, chosen at random. The script (\texttt{bn\_trail\_sim.py}, seed 12345) runs in about ten seconds.

\paragraph{Score fidelity and equivalence.} On chains of 1 to 10 edges with flip probabilities 0.05, 0.1, 0.2 and 0.3, the trail score matches the exact endpoint correlation to within $1.1 \times 10^{-15}$, as Proposition~1 states. Over 200 random strength pairs, a chain and a fork have identical joint distributions (largest difference 0.0), and $A$ is independent of $C$ given $B$ up to $1.1 \times 10^{-16}$.

\paragraph{d-separation pruning is exact.} Dropping evidence on blocked trails changed the posterior by at most $4.4 \times 10^{-16}$ in every trial of both regimes. Only about two evidence nodes per trial are active (2.00 in the first regime, 1.84 in the second), so thresholding has little to remove: the mean number kept falls from 2.00 to 0.61, and from 1.84 to 0.59, as $\tau$ rises to 0.5.

\paragraph{Trail thresholding against random pruning.} At every $\tau$ above 0, the mean error of trail-score pruning is below that of random pruning of the same size. In total variation the random baseline is worse by a factor of 1.4 to 13.3 (1.39 to 9.82 on 12-node trees, 1.44 to 13.28 on 14-node trees), and in KL divergence by a factor of 2 to 536. The advantage is largest at small thresholds, where trail scores correctly identify the few weakest trails to drop.

\begin{figure}[ht]
\centering
\includegraphics[width=\linewidth]{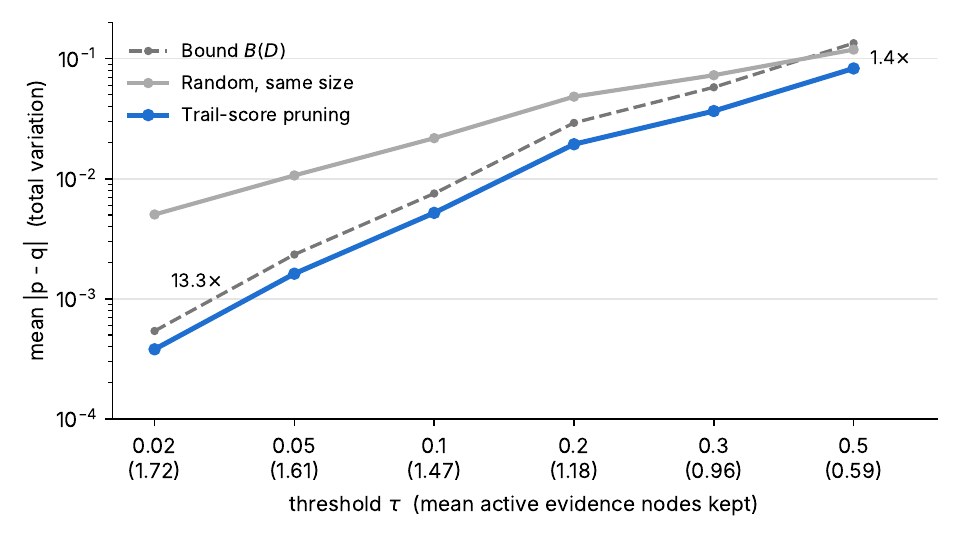}
\caption{Mean posterior error (total variation) against threshold $\tau$ on 14-node trees (3,000 trials, seed 12345), on a log scale. Factors give random error divided by trail-score error. With $\tau=0$ (d-separation only) the error is zero and is not shown.}
\label{fig:err}
\end{figure}

\paragraph{Bound check.} Table~\ref{tab:bound} splits all trial--threshold pairs with $\tau$ above 0 by whether every dropped trail was alone in its branch of $Y$, the condition of Proposition~\ref{prop:bound}. Pairs where nothing was dropped satisfy the condition vacuously.

\begin{table}[ht]
\centering
\caption{Violations of $|p-q| \le B(D)$ by whether the isolation condition holds.}
\label{tab:bound}
\begin{tabular}{llrrr}
\toprule
Regime & Dropped trails & Pairs ($\tau>0$) & Violations & Largest error/bound \\
\midrule
12-node trees & each alone in its branch & 19,681 & 0 & below 1 \\
12-node trees & share a branch & 4,319 & 5 & 1.085 \\
14-node trees & each alone in its branch & 15,606 & 0 & below 1 \\
14-node trees & share a branch & 2,394 & 0 & below 1 \\
\bottomrule
\end{tabular}
\end{table}

There are no violations when Proposition~\ref{prop:bound} applies, and all five violations occur where dropped trails share a branch. In the isolated case the largest ratio is 0.999998 or higher, which confirms the bound is nearly tight for weak trails. Averaged over trials, the bound exceeds the actual error by a factor of 1.4 to 1.9, so it is informative rather than loose.

\section{Benchmark experiments}\label{sec:bench}

Keeping the top half of the active evidence by trail score gives lower error than random pruning on all seven benchmark networks, but it is not clearly better than a mutual-information ranking, and it saves no time.

\paragraph{Setup.} We use seven networks from the bnlearn repository, loaded through pgmpy 1.1.2 (Table~\ref{tab:setup}). They contain colliders and loops, so the theory above does not apply and the trail score is used as a heuristic. Each query draws a random outcome node $Y$ and $m$ random evidence nodes, with values taken from a forward sample of the network. Here $m$ is $0.4(n-1)$, rounded and limited to between 4 and 14. We keep queries with at least four active evidence nodes after exact d-separation. The seed is 12345, and the script is \texttt{bench\_bnlearn.py}. The relevance of an evidence node $e$ is $\mathrm{KL}\bigl(P(Y \mid E)\,\Vert\,P(Y \mid E \setminus e)\bigr)$, computed by exact variable elimination.

\begin{table}[ht]
\centering\small
\caption{Benchmark networks and queries.}\label{tab:setup}
\begin{tabular}{lrrrrr}
\toprule
Network & Nodes & Edges & Evidence per query & Active evidence (mean) & Queries \\
\midrule
Asia & 8 & 8 & 4 & 4.00 & 300 \\
Sachs & 11 & 17 & 4 & 4.00 & 300 \\
Child & 20 & 25 & 8 & 6.28 & 300 \\
Insurance & 27 & 52 & 10 & 9.20 & 300 \\
Alarm & 37 & 46 & 14 & 7.92 & 300 \\
Hailfinder & 56 & 66 & 14 & 9.37 & 100 \\
Andes & 223 & 338 & 14 & 12.29 & 100 \\
\bottomrule
\end{tabular}
\end{table}

\paragraph{Configuration.} All experiments ran on one machine with a 2-core Intel Xeon at 2.1 GHz, Python 3.13, pgmpy 1.1.2 (variable elimination), networkx 3.6.1, numpy 2.5.3, scipy 1.18.1, scikit-learn 1.9.1 and shap 0.53.0, in a single process. The networks are the pgmpy copies of the bnlearn repository models, used unchanged, and evidence values are sampled from the network itself. We re-ran the whole benchmark from the same seed and obtained identical KL values (largest difference $3\times10^{-17}$). Each query is timed once, so the timing tables report the mean and standard deviation over queries, and run-to-run differences in the means were up to 2 times on the smallest networks (for example, exact inference on Asia took 0.79 ms in the first run and 0.66 ms in the second).

\paragraph{Methods.} After exact d-separation the active evidence is ranked, and the top 25\%, 50\% or 75\% is kept. The rankings are: the trail score (ours, the product of edge maximal correlations); two variants from the proposal that use normalized mutual information on edges, composed by product or by minimum; the marginal mutual information $I(Y;e)$, which needs a pairwise inference per evidence node; hop distance from $Y$, nearest first; random choice of the same size, averaged over three draws; and a reference ranking by true leave-one-out relevance. The reference is not an upper bound, because leave-one-out relevance ignores interactions among the kept nodes. Error is the KL divergence from the exact posterior of $Y$ to the pruned one, and we report 50\% kept unless stated.

\paragraph{d-separation is exact.} Across all seven networks the largest KL divergence after d-separation pruning is $2.1\times10^{-16}$, and the reduction never needed a fallback.

\begin{figure}[ht]
\centering
\includegraphics[width=\linewidth]{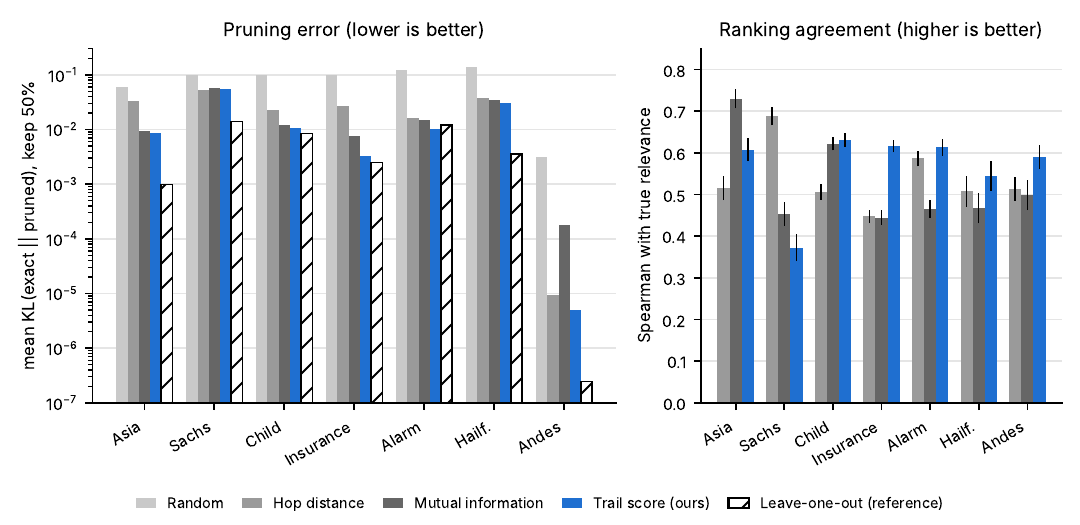}
\caption{Left: mean KL divergence from the exact posterior when keeping 50\% of the active evidence, on a log scale. Right: Spearman correlation between each score and the true relevance, with standard errors. Seed 12345; 300 queries per network, 100 for Hailfinder and Andes.}\label{fig:bench}
\end{figure}

\begin{table}[ht]
\centering\small
\caption{Mean KL divergence from the exact posterior when keeping 50\% of the active evidence.}\label{tab:kl}
\begin{tabular}{lccccc}
\toprule
Network & Random & Hop distance & Mutual information & Trail score & Reference \\
\midrule
Asia & 0.0609 & 0.0326 & 0.0093 & 0.0087 & $9.9\times10^{-4}$ \\
Sachs & 0.0988 & 0.0523 & 0.0573 & 0.0557 & 0.0140 \\
Child & 0.0983 & 0.0226 & 0.0123 & 0.0104 & 0.0085 \\
Insurance & 0.0994 & 0.0273 & 0.0076 & 0.0032 & 0.0025 \\
Alarm & 0.1220 & 0.0164 & 0.0148 & 0.0102 & 0.0119 \\
Hailfinder & 0.1363 & 0.0382 & 0.0344 & 0.0303 & 0.0036 \\
Andes & 0.0031 & $9.3\times10^{-6}$ & $1.8\times10^{-4}$ & $5\times10^{-6}$ & $2.5\times10^{-7}$ \\
\bottomrule
\end{tabular}
\end{table}

\paragraph{Against random and hop distance.} Trail pruning has lower mean error than random pruning on all seven networks (Table~\ref{tab:kl}), by a factor of 1.8 (Sachs) to 31 (Insurance) on the six networks where errors are not negligible. The paired 95\% bootstrap interval of the difference excludes zero on every network, and trail pruning is better than random on 75\% to 95\% of queries on Asia, Sachs, Child, Insurance and Alarm. Against hop distance, trail pruning is clearly better on Asia, Child, Insurance and Hailfinder. The differences on Sachs and Alarm are not clear, and on Andes both errors are below $10^{-5}$.

\paragraph{Against mutual information.} Trail pruning has lower mean error than the mutual-information ranking on all seven networks, but the paired interval excludes zero only narrowly on Asia (a difference of 0.0005), Insurance and Andes, and includes zero on Sachs, Child, Alarm and Hailfinder. We therefore do not claim it is better. The trail score is a closed form in edge strengths computed once per network, while the mutual-information score needs a pairwise inference per evidence node, but the timings below show no speed advantage for our implementation.

\paragraph{Ranking agreement.} Table~\ref{tab:sp} gives the Spearman correlation between each score and the true relevance, averaged over queries. The trail score reaches 0.55 to 0.63 on six networks and 0.37 on Sachs. It is higher than the mutual-information score on Insurance, Alarm, Hailfinder and Andes (by 0.08 to 0.18), about equal on Child, and lower on Asia and Sachs. Hop distance ranks best on Sachs.

\begin{table}[ht]
\centering\small
\caption{Spearman correlation between each score and the true relevance.}\label{tab:sp}
\begin{tabular}{lccc}
\toprule
Network & Hop distance & Mutual information & Trail score \\
\midrule
Asia & 0.52 & 0.73 & 0.61 \\
Sachs & 0.69 & 0.45 & 0.37 \\
Child & 0.51 & 0.62 & 0.63 \\
Insurance & 0.45 & 0.44 & 0.62 \\
Alarm & 0.59 & 0.47 & 0.61 \\
Hailfinder & 0.51 & 0.47 & 0.55 \\
Andes & 0.51 & 0.50 & 0.59 \\
\bottomrule
\end{tabular}
\end{table}

\paragraph{Composition and decisions.} Composing normalized mutual information by minimum gives clearly higher error than the maximal-correlation product on Asia, Child, Insurance and Alarm (the paired intervals exclude zero), and no clear difference on Sachs. On Hailfinder the minimum is lower (0.017 against 0.030), and we did not test that difference. Composing normalized mutual information by product is close to the maximal-correlation product. Keeping half of the active evidence by trail score preserves the most likely state of $Y$ in 95\% to 99\% of queries, against 78\% to 95\% for random pruning.

\paragraph{Cost.} Table~\ref{tab:time} gives the mean time in milliseconds per query with the standard deviation over queries, except the last column, which is paid once per network. Exact inference is fast because pgmpy's variable elimination already discards d-separated and barren nodes. Trail scoring, in unoptimized Python, costs more than the inference it saves on six of the seven networks and about the same on Asia (the saving is the difference between the Exact and Half columns), and pruning to half of the evidence cuts inference time by only 0.1 to 4.7 ms. The standard deviations are large relative to the means, so differences under a factor of two on the small networks are not reliable. We therefore show no speedup. What the score offers on these networks is the trail that explains each piece of evidence, not run time.

\begin{table}[ht]
\centering\small
\caption{Time in milliseconds, mean (standard deviation over queries). Exact: exact inference with all evidence. Trail: trail scoring. MI: mutual-information scoring. Half: inference with half of the active evidence kept. Edges: edge strengths, computed once per network.}\label{tab:time}
\begin{tabular}{lccccc}
\toprule
Network & Exact & Trail & MI & Half & Edges (once) \\
\midrule
Asia & 0.66 (0.14) & 0.08 (0.07) & 0.06 (0.35) & 0.58 (0.14) & 10 \\
Sachs & 0.68 (0.18) & 0.34 (0.40) & 0.08 (0.40) & 0.49 (0.16) & 10 \\
Child & 1.21 (0.32) & 0.93 (1.11) & 0.56 (1.24) & 0.79 (0.26) & 17 \\
Insurance & 2.04 (0.51) & 24.3 (28.0) & 1.61 (3.31) & 1.53 (0.66) & 69 \\
Alarm & 2.12 (0.59) & 2.44 (2.40) & 2.43 (3.54) & 1.63 (0.72) & 48 \\
Hailfinder & 3.04 (1.11) & 6.23 (4.21) & 11.1 (8.84) & 2.35 (1.42) & 84 \\
Andes & 20.4 (5.47) & 89.9 (82.0) & 107 (60.2) & 15.7 (6.22) & 2,990 \\
\bottomrule
\end{tabular}
\end{table}

\section{Learned calibration}\label{sec:learned}

Models trained on trail, graph and information features and tested on networks they never saw do not beat the hand-crafted trail score or mutual information, but trail features carry most of what the models use.

\paragraph{Setup.} The proposal asks whether supervised models can calibrate trail scores better than a hand-crafted formula. For each active evidence node we compute 37 features in four groups: 16 trail features (for example the best product score over active trails with and without colliders, the sum of scores over trails, the length of the shortest active trail, and the number of other evidence nodes lying on trails to $Y$), 13 graph features (hop distance, degrees, betweenness, PageRank and Markov-blanket relations), 6 information features (mutual information with $Y$, entropies and cardinalities) and 2 query features (evidence-set size and number of active nodes). The label is the exact leave-one-out relevance, and the regression target is $\log_{10}(\text{relevance}+10^{-6})$. We built a new query set with the evidence size drawn from 4 to $\min(14,n-1)$: 2,450 queries (400 on each of the five smaller networks, 300 on Hailfinder and 150 on Andes; seed 4242), which give 15,374 rows. The models are ridge regression, a random forest, gradient boosting and a two-layer MLP, with hyperparameters fixed in advance. Each is trained on six networks and tested on the seventh (leave-one-network-out), and its predictions rank the held-out evidence exactly as the other scores do. We designated gradient boosting as the primary model before seeing results.

\begin{figure}[ht]
\centering
\includegraphics[width=\textwidth]{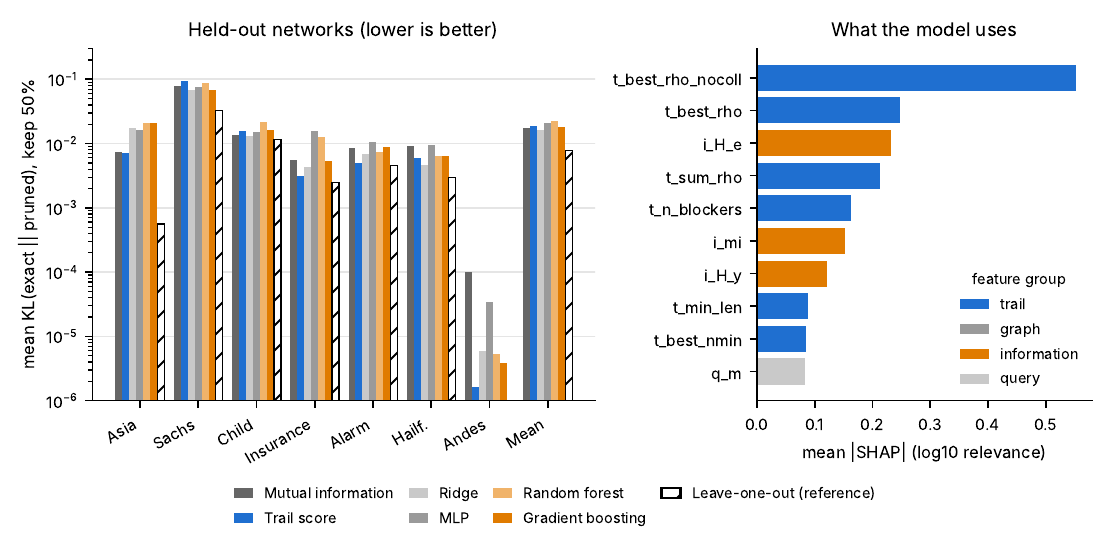}
\caption{Left: mean KL($P(Y\mid E)\,\|\,$pruned) keeping 50\% of the active evidence, per held-out network (log scale). The leave-one-out reference is not an upper bound. Right: ten largest mean absolute SHAP values of the gradient boosting model.}\label{fig:learned}
\end{figure}

\begin{table}[ht]
\centering\small
\caption{Mean KL at 50\% kept for models trained on the other six networks. Mean is the equal-weight average over networks. Ref: leave-one-out reference.}\label{tab:learned}
\begin{tabular}{lrrrrrrr}
\toprule
Network & MI & Trail & Ridge & MLP & RF & GB & Ref \\
\midrule
Asia & .00732 & .00709 & .0175 & .0162 & .0206 & .0210 & .00056 \\
Sachs & .0780 & .0929 & .0674 & .0769 & .0861 & .0689 & .0325 \\
Child & .0137 & .0154 & .0131 & .0151 & .0216 & .0166 & .0116 \\
Insurance & .00554 & .00317 & .00434 & .0159 & .0125 & .00545 & .00247 \\
Alarm & .00858 & .00499 & .00681 & .0107 & .00752 & .00875 & .00458 \\
Hailfinder & .00929 & .00594 & .00462 & .00945 & .00652 & .00636 & .00299 \\
Andes & $10^{-4}$ & $1.7\times10^{-6}$ & $5.9\times10^{-6}$ & $3.5\times10^{-5}$ & $5.3\times10^{-6}$ & $3.9\times10^{-6}$ & $8.6\times10^{-7}$ \\
\midrule
Mean & .0175 & .0185 & .0163 & .0206 & .0221 & .0181 & .00781 \\
\bottomrule
\end{tabular}
\end{table}

\paragraph{Results.} Averaged over the seven networks with equal weight, the mean KL at 50\% kept is 0.0163 for ridge regression, 0.0175 for mutual information, 0.0181 for gradient boosting, 0.0185 for the trail score, 0.0206 for the MLP and 0.0221 for the random forest (Table~\ref{tab:learned}). Random pruning gives 0.083 and the leave-one-out reference 0.0078. The paired 95\% intervals for the difference between gradient boosting and the trail score ($-0.0004$, interval $-0.0042$ to $+0.0038$) and between gradient boosting and mutual information ($+0.0006$, interval $-0.0029$ to $+0.0046$) include zero, as do those for ridge regression. The random forest and the MLP are worse than mutual information (the intervals exclude zero, narrowly for the MLP). On single networks, gradient boosting beats the trail score clearly only on Sachs (by 0.024) and is worse on Alarm (by 0.004).

The Spearman correlation with the true relevance, averaged over networks, is 0.55 for ridge regression and gradient boosting, 0.54 for the random forest, 0.53 for both the trail score and mutual information, and 0.43 for the MLP. Keeping half of the evidence preserves the most likely state of $Y$ in 97\% of queries for gradient boosting and for the trail score alike. We therefore find no support for the hypothesis that a learned model beats the hand-crafted score. All rankings remain well short of the reference: gradient boosting is 0.010 above it (interval 0.006 to 0.015).

\paragraph{Do trail features add value?} We ran gradient boosting on subsets of the feature groups (Table~\ref{tab:ablate}).

\begin{table}[ht]
\centering\small
\caption{Gradient boosting on feature subsets, leave-one-network-out.}\label{tab:ablate}
\begin{tabular}{lcc}
\toprule
Features & Mean KL, 50\% kept & Spearman \\
\midrule
All 37 & 0.0181 & 0.55 \\
Trail + query & 0.0192 & 0.53 \\
Graph + query & 0.0340 & 0.36 \\
Information + query & 0.0187 & 0.47 \\
Graph + information + query & 0.0200 & 0.46 \\
Trail + graph + query & 0.0189 & 0.52 \\
Trail + information + query & 0.0166 & 0.55 \\
\bottomrule
\end{tabular}
\end{table}

Trail features alone are nearly as good as all features (0.0192 against 0.0181) and far better than graph features alone (0.0340). The paired interval for trail plus query against graph plus query is $-0.0148$ ($-0.0193$ to $-0.0099$). Adding trail features to graph features lowers the error by 0.0151 (interval $-0.0193$ to $-0.0112$). Adding them to graph and information features lowers it by 0.0018, but that interval ($-0.0040$ to $+0.00004$) just reaches zero. Trail features therefore add clear predictive value over generic graph features, and a smaller, unconfirmed amount over information features.

\paragraph{What the model uses.} SHAP values for a gradient boosting model fitted to all rows show that trail features account for 64\% of the mean absolute attribution, information features 23\%, graph features 9\% and query features 4\%. Table~\ref{tab:shap} lists the ten largest. Directions are conditional on the other features, so the sign for the normalized-MI score, which overlaps the product scores, is not interpretable alone.

\begin{table}[ht]
\centering\small
\caption{Ten largest mean absolute SHAP values.}\label{tab:shap}
\begin{tabular}{llcl}
\toprule
Feature & Meaning & SHAP & Higher value \\
\midrule
\texttt{t\_best\_rho\_nocoll} & Best product score, active trails without collider & 0.55 & raises \\
\texttt{t\_best\_rho} & Best product score over active trails & 0.25 & raises \\
\texttt{i\_H\_e} & Entropy of the evidence node & 0.23 & raises \\
\texttt{t\_sum\_rho} & Sum of product scores over active trails & 0.21 & raises \\
\texttt{t\_n\_blockers} & Other evidence nodes on trails to $Y$ & 0.16 & lowers \\
\texttt{i\_mi} & Mutual information with $Y$ & 0.15 & raises \\
\texttt{i\_H\_y} & Entropy of $Y$ & 0.12 & raises \\
\texttt{t\_min\_len} & Length of the shortest active trail & 0.09 & lowers \\
\texttt{t\_best\_nmin} & Best minimum normalized-MI score & 0.09 & lowers \\
\texttt{q\_m} & Number of evidence nodes & 0.08 & lowers \\
\bottomrule
\end{tabular}
\end{table}

\paragraph{Cost.} Feature extraction takes 0.3 to 160 ms per query, more than exact inference on five of the seven networks (0.7 to 20 ms). Prediction takes 0.01 to 0.02 ms. Training, summed over the seven folds, took 0.3 s for ridge regression, 12 s for the MLP, 56 s for gradient boosting and 122 s for the random forest. Learning therefore adds a one-off training cost and does not make scoring cheaper.

\section{Related work}

Exact structural pruning is well established, and our contribution is an approximate second stage that uses the strength of a trail rather than its existence.

\paragraph{Relevance and pruning.} Pearl (1988) introduced d-separation, and Geiger, Verma and Pearl (1990) gave an algorithmic characterization of the independencies it encodes. Shachter's Bayes-Ball algorithm (1998) finds the irrelevant and non-requisite nodes for a query by a graph traversal in linear time. Baker and Boult (1990) and Lin and Druzdzel (1997) prune networks before inference to reduce computation. These methods are exact. Ours keeps them as the first stage and adds a threshold on trail strength.

\paragraph{Approximation by removing weak dependences.} Kj{\ae}rulff (1994) removes links from the moralized graph when their conditional mutual information is small, which splits large cliques, and bounds the error through the KL divergence of the clique potentials. He reports state-space reductions of 34\% to 97\% on real networks. Choi, Chan and Darwiche (2005) delete edges, replacing the deleted parent by its posterior given the evidence, and bound the KL error using the entropy of the deleted parent. Both methods change the network to make inference cheaper, and both need inference-based quantities to choose what to remove. Our method leaves the network unchanged and prunes the evidence by a threshold on a closed-form trail score. The approaches are complementary, and we have not compared them.

\paragraph{Sensitivity and strength of influence.} Chan and Darwiche (2002) study how a query result changes when network parameters change. That asks which parameters matter, whereas we ask which evidence can be omitted. We use maximal correlation as the edge strength because it composes by product along a trail. Link-strength and connection-strength measures between nodes, mostly based on mutual information, have been proposed (Boerlage, 1994; Ebert-Uphoff, 2007). We have not checked whether these compose along trails as ours does, and the experiments above include a mutual-information-only ranking. The GeNIe software also draws a ``strength of influence'' on each arc from the child's conditional probability table, but its documentation describes it as a tool for visual analysis and does not say that it composes along paths or prunes evidence.

\paragraph{Dependence along chains.} The decay of dependence along a Markov chain is classical: the data processing inequality (Cover and Thomas, 2006) and the multiplicative contraction of maximal correlation (R\'enyi, 1959; Witsenhausen, 1975). We use these facts as a pruning score and test the resulting error bound against exact inference.

\paragraph{Exact and approximate inference.} Junction-tree inference (Lauritzen and Spiegelhalter, 1988) is exact but exponential in treewidth. Loopy belief propagation (Murphy, Weiss and Jordan, 1999) and sampling are approximate. Koller and Friedman (2009) and Darwiche (2009) cover both. These are the baselines for the loopy-network evaluation planned as future work.

\paragraph{What the trail score adds.} The closest existing tools answer different questions. d-separation and Bayes-Ball (Shachter, 1998; Meekes, Renooij and van der Gaag) say whether evidence can matter, not how much. Mutual information and the link-strength measures (Boerlage, 1994; Ebert-Uphoff, 2007) measure how much a pair of nodes or a single edge carries, but they are not defined for the evidence that remains after d-separation, in the context of the other observations. The trail score does three things that these do not. It is defined on the active trails of the query, so it changes with the evidence set. It composes along a trail by a product, and in collider-free binary networks the product equals the exact endpoint correlation (Proposition~1). And it names the trail that explains each kept or dropped item. What it does not add, in our experiments, is accuracy: it is statistically indistinguishable from a mutual-information ranking on the mean error, and it carries an error bound only in the collider-free case under the isolation condition.

\paragraph{Relevance of evidence.} Meekes, Renooij and van der Gaag define sets of given evidence whose value or status may change a posterior, and compute them exactly with Bayes-Ball. Their method prunes by structure only and does not rank the remaining evidence by strength. Champion and Elkan rank variables, not evidence items, by the symmetric KL divergence between posteriors under two evidence sets, and use the ranking to declutter a visualization rather than to prune.

\paragraph{Learned inference.} Douglas et al.\ train one neural network on samples from a Bayesian network, with random masking of observed nodes, to approximate marginals under any evidence. It amortizes inference and does not rank or prune evidence. Our learned-calibration experiment instead predicts per-evidence relevance from trail features.

\paragraph{Still to complete.} Our search covered the sources cited here and was not systematic. A search of Google Scholar, Semantic Scholar and arXiv by the authors, and a review of graph neural networks for probabilistic reasoning, remain to be done. To our knowledge we found no work that scores evidence by the product of edge strengths along active trails, but we cannot rule it out.

\section{Discussion, limitations and future work}

In collider-free binary networks, trail strength is an exact and cheap measure of how much evidence tells us about the outcome, and thresholding it removes evidence with small, boundable error. The synthetic results support the proposal's second hypothesis (pruning keeps the error small) and, under an isolation condition, part of its fourth (a provable bound on trees). The benchmarks bear on the first hypothesis, since trail scores agree only moderately with true relevance (Spearman 0.37 to 0.63), and on the second, since pruning beats random but not clearly a mutual-information ranking. The learned-calibration experiment tests the third (a learned model beating the hand-crafted score) and does not support it. The loopy-graph part of the fourth, a bound on general networks, is still open.

\paragraph{Limitations.}
\begin{itemize}[leftmargin=*]
\item The theory assumes binary variables, uniform marginals, symmetric edges and no colliders. On general networks the trail score is a heuristic with no error bound.
\item In the synthetic experiments only about two evidence nodes per trial remain active after d-separation, and the only baseline is random pruning.
\item Proposition~\ref{prop:bound} needs each dropped trail to be alone in its branch. With shared branches the bound failed in 5 of 4,319 trial--threshold pairs on 12-node trees, by up to 8.5\%, and we have no valid replacement bound.
\item On the benchmarks, trail pruning cannot be told apart from a mutual-information ranking on most networks, and its agreement with true relevance is moderate (Spearman 0.37 to 0.63).
\item Scoring costs more than the inference it saves on six of the seven benchmark networks, so we show no speedup. Timings come from one run on one 2-core machine, with a standard deviation across queries that is large relative to the mean.
\item Queries are random draws with at least four active evidence nodes, from one seed, 100 to 300 per network. Trails are capped at 7 edges and 300 per evidence node, and ground truth needs exact inference, which limits the networks we can evaluate.
\item Learned models were trained on six networks and tested on a seventh with fixed hyperparameters and one seed. Their paired differences from the trail score and mutual information are within noise, so a different model class or more training networks could change the outcome.
\end{itemize}

\paragraph{Future work.}
\begin{itemize}[leftmargin=*]
\item \textbf{Shared branches.} Derive a bound that accounts for dependence among dropped trails that share a branch of $Y$.
\item \textbf{General networks.} Derive error bounds for colliders and loopy graphs, where the product score is only a heuristic.
\item \textbf{Stronger baselines and cost.} Compare with loopy belief propagation, likelihood weighting and edge deletion, and test networks where exact inference is slow enough that pruning saves time.
\item \textbf{Better learned models.} Try models that use the graph directly (for example graph neural networks), more training networks, and a loss aimed at the ranking rather than at relevance values.
\end{itemize}

\section*{Reproducibility and code availability}
All results were produced by the scripts below, each with a fixed seed.
\begin{itemize}[leftmargin=*,itemsep=1pt]
\item \texttt{bn\_trail\_sim.py}: synthetic results (seed 12345, about ten seconds).
\item \texttt{verify\_prop3.py}, \texttt{bn\_prop3\_exhaustive.py}: exhaustive checks of Proposition~\ref{prop:bound}.
\item \texttt{bench\_bnlearn.py}: benchmark (seed 12345; 300 queries on Asia, Sachs, Child, Insurance and Alarm, 100 on Hailfinder and Andes).
\item \texttt{learn\_build.py}, \texttt{learn\_eval.py}, \texttt{learn\_analyze.py}: learned calibration (seed 4242 for the queries).
\end{itemize}
Software versions are given in the benchmark configuration. Networks are loaded with pgmpy's example-model loader, with no preprocessing. The code, the raw result files and a README with the exact commands are in the public repository \url{https://github.com/sahil1290-cm/trail-weighted-bn}.


\begin{thebibliography}{99}\small
\setlength{\itemsep}{2pt}
\bibitem{} Baker, M. and Boult, T. E. (1990). Pruning Bayesian networks for efficient computation. \emph{Proceedings of UAI}, 225--232.
\bibitem{} Chan, H. and Darwiche, A. (2002). When do numbers really matter? \emph{Journal of Artificial Intelligence Research}, 17, 265--287.
\bibitem{} Choi, A., Chan, H. and Darwiche, A. (2005). On Bayesian network approximation by edge deletion. \emph{Proceedings of UAI}, 128--135.
\bibitem{} Cooper, G. F. (1990). The computational complexity of probabilistic inference using Bayesian belief networks. \emph{Artificial Intelligence}, 42(2--3), 393--405.
\bibitem{} Cover, T. M. and Thomas, J. A. (2006). \emph{Elements of Information Theory}, 2nd edition. Wiley.
\bibitem{} Darwiche, A. (2009). \emph{Modeling and Reasoning with Bayesian Networks}. Cambridge University Press.
\bibitem{} Ebert-Uphoff, I. (2007). Measuring connection strengths and link strengths in discrete Bayesian networks. Technical report, Georgia Institute of Technology, 29 January 2007.
\bibitem{} Geiger, D., Verma, T. and Pearl, J. (1990). Identifying independence in Bayesian networks. \emph{Networks}, 20(5), 507--534.
\bibitem{} Kj{\ae}rulff, U. (1994). Reduction of computational complexity in Bayesian networks through removal of weak dependences. \emph{Proceedings of UAI}, 374--382.
\bibitem{} Koller, D. and Friedman, N. (2009). \emph{Probabilistic Graphical Models: Principles and Techniques}. MIT Press.
\bibitem{} Lauritzen, S. L. and Spiegelhalter, D. J. (1988). Local computations with probabilities on graphical structures and their application to expert systems. \emph{Journal of the Royal Statistical Society, Series B}, 50(2), 157--224.
\bibitem{} Lin, Y. and Druzdzel, M. J. (1997). Computational advantages of relevance reasoning in Bayesian belief networks. \emph{Proceedings of UAI}, 342--350.
\bibitem{} Boerlage, B. (1994). Link strength in Bayesian networks. Technical report TR-94-17, University of British Columbia.
\bibitem{} Murphy, K. P., Weiss, Y. and Jordan, M. I. (1999). Loopy belief propagation for approximate inference: an empirical study. \emph{Proceedings of UAI}.
\bibitem{} Champion, C. and Elkan, C. Visualizing the consequences of evidence in Bayesian networks. arXiv:1707.00791.
\bibitem{} Douglas, L., Zarov, I., Gourgoulias, K., Lucas, C., Hart, C., Baker, A., Sahani, M., Perov, Y. and Johri, S. A universal marginalizer for amortized inference in generative models. arXiv:1711.00695.
\bibitem{} Meekes, M., Renooij, S. and van der Gaag, L. C. Relevance of evidence in Bayesian networks. Utrecht University (ECSQARU 2015 manuscript).
\bibitem{} Pearl, J. (1988). \emph{Probabilistic Reasoning in Intelligent Systems}. Morgan Kaufmann.
\bibitem{} R\'enyi, A. (1959). On measures of dependence. \emph{Acta Mathematica Academiae Scientiarum Hungaricae}, 10, 441--451.
\bibitem{} Shachter, R. D. (1998). Bayes-Ball: the rational pastime. \emph{Proceedings of UAI}.
\bibitem{} Witsenhausen, H. S. (1975). On sequences of pairs of dependent random variables. \emph{SIAM Journal on Applied Mathematics}, 28(1), 100--113.
\end{thebibliography}
\end{document}